\documentclass{article}
\usepackage{PRIMEarxiv}
\usepackage{setspace}
\usepackage{algpseudocode}
\usepackage{float}
\usepackage[numbers]{natbib}
\usepackage{algorithm}
\usepackage{enumitem}
\usepackage{xspace} 
\usepackage{tikz}
\usetikzlibrary{arrows.meta,positioning,fit,calc}
\usepackage{booktabs}
\usepackage{siunitx}
\usepackage{bm}        
\usepackage{mathrsfs} 

\usepackage{amsmath,amssymb}  
\usepackage{amsthm}          
\newtheorem{theorem}{Theorem}
\newtheorem{lemma}[theorem]{Lemma}

\theoremstyle{definition}

\usepackage[indLines]{algpseudocodex}  
\usepackage{amssymb}
\usepackage[utf8]{inputenc}
\usepackage[T1]{fontenc} 
\usepackage{hyperref}
\hypersetup{hidelinks}
\usepackage[symbol]{footmisc}
\renewcommand{\thefootnote}{\fnsymbol{footnote}}
\usepackage{url}          
\usepackage{booktabs}       
\usepackage{amsfonts}       
\usepackage{nicefrac}       
\usepackage{microtype}  
\usepackage{lipsum}
\usepackage{fancyhdr}   
\usepackage{graphicx}    
\graphicspath{{media/}}  
\usepackage{hyperref}
\usepackage{amsmath,amssymb,amsfonts,amsthm}
\usepackage{bm}
\usepackage{mathrsfs}
\usepackage[numbers]{natbib}
\usepackage{enumitem}
\usepackage{parskip}
\usepackage{xcolor}

\newcommand{\cF}{\mathcal{F}}
\newcommand{\cH}{\mathcal{H}}
\newcommand{\cE}{\mathcal{E}}
\newcommand{\cR}{\mathcal{R}}
\newcommand{\cD}{\mathcal{D}}
\newcommand{\cN}{\mathcal{N}}        
\newcommand{\Fclass}{\mathcal{F}_{K,M,h,P}}   

\newcommand{\E}{\mathbb{E}}
\newcommand{\N}{\mathbb{N}}
\newcommand{\R}{\mathbb{R}}

\newcommand{\Rad}{\mathfrak{R}}      
\newcommand{\wtO}[1]{\widetilde{\mathcal{O}}\!\bigl(#1\bigr)}

\theoremstyle{plain}

\usepackage{hyperref}
\usepackage{amsfonts,amsthm}
\usepackage{xcolor}

\theoremstyle{remark}

\makeatletter
\@ifpackageloaded{amsmath}{}{\usepackage{amsmath}}
\@ifpackageloaded{amssymb}{}{\usepackage{amssymb}}
\@ifpackageloaded{amsthm}{}{\usepackage{amsthm}}
\makeatother

\providecommand{\E}{\mathbb{E}}   

\title{Anchor–MoE: A Mean-Anchored Mixture of Experts for Probabilistic Regression
}

\author{
  Baozhuo Su
  \thanks{Work is partially done at Linkconn Electronics Co., Ltd}
  \hspace{0.1em}\thanks{Corresponding Author}\\
  Institut Polytechnique de Paris \\
  \texttt{baozhuo.su@ip-paris.fr} \\
  \texttt{} \\
  \and
   \textbf{Zhengxian Qu} \\
  Aerie Intelligent Technology Corp\\
  \texttt{squ@aerieintech.com} \\
}

\begin{document}
\maketitle
\onehalfspacing
\renewcommand{\thefootnote}{\arabic{footnote}}

\begin{abstract}
Regression under uncertainty is fundamental across science and engineering. We present an \emph{Anchored Mixture of Experts} (Anchor--MoE), a model that handles both probabilistic and point regression. For simplicity, we use a tuned gradient-boosting model to furnish the \emph{anchor mean}; however, any off-the-shelf point regressor can serve as the anchor. The anchor prediction is projected into a latent space, where a learnable metric--window kernel scores locality and a soft router dispatches each sample to a small set of mixture-density-network experts; the experts produce a heteroscedastic correction and predictive variance. We train by minimizing negative log-likelihood, and on a disjoint calibration split fit a post-hoc linear map on predicted means to improve point accuracy. On the theory side, assuming a H\"older smooth regression function of order~$\alpha$ and fixed Lipschitz partition-of-unity weights with bounded overlap, we show that Anchor--MoE attains the minimax-optimal $L^2$ risk rate $\mathcal{O}\!\big(N^{-2\alpha/(2\alpha+d)}\big)$. In addition, the CRPS test generalization gap scales as $\widetilde{\mathcal{O}}\!\Big(\sqrt{(\log(Mh)+P+K)/N}\Big)$; it is logarithmic in $Mh$ and scales as the square root in $P$ and $K$. Under bounded-overlap routing, $K$ can be replaced by $k$, and any dependence on a latent dimension is absorbed into $P$. Under uniformly bounded means and variances, an analogous $\widetilde{\mathcal{O}}\!\big(\sqrt{(\log(Mh)+P+K)/N}\big)$ scaling holds for the test NLL up to constants. Empirically, across standard UCI regressions, Anchor--MoE consistently matches or surpasses the strong NGBoost baseline in RMSE and NLL; on several datasets it achieves new state-of-the-art probabilistic regression results on our benchmark suite. Code is available at \url{https://github.com/BaozhuoSU/Probabilistic_Regression}.
\end{abstract}


\section*{1. Introduction}
\noindent
Regression is a cornerstone of machine learning: given covariates $\mathbf{X}$ and a real-valued response $Y$, the goal under mean-squared-error loss is to estimate the conditional expectation $f^\star(x)=\mathbb{E}[Y \mid \mathbf{X}=x]$, which is the population risk minimizer. Regression methods are ubiquitous in modern research, powering applications from climate forecasting~\cite{weatherForcasting} and protein engineering~\cite{proteinEngineering} to chronic disease prognosis~\cite{chronicDisease}.
\par
Most machine learning approaches cast regression as learning a deterministic mapping and optimize mean-squared error, effectively estimating \(\mathbb{E}[\mathbf{Y}\mid \mathbf{X}]\). However, Kendall and Gal~\cite{kendall2017uncertainties} show that explicitly modeling the full predictive distribution, especially heteroscedastic noise—can improve point accuracy by weighting residuals with learned uncertainty. In this probabilistic regression view we learn \(p(\mathbf{Y}\mid \mathbf{X})\) rather than only its mean, enabling calibrated uncertainty quantification and better downstream decisions (e.g., financial risk management), with strong empirical performance~\cite{finance}.
\par
Building on these practical benefits, a range of probabilistic regression families has been proposed, including Bayesian/uncertainty-aware neural approaches~\cite{DLRegression,kendall2017uncertainties}, tree-based probabilistic ensembles~\cite{treeBased}, and distributional generalized additive models (GAMLSS)~\cite{generalized}. While all aim to return full predictive distributions, they come with different trade-offs: deep and ensemble methods can be computationally intensive and often reduce interpretability; GAMLSS requires distributional and link-function specification and can be challenging to scale in very high-dimensional settings; and on some datasets, certain probabilistic models may prioritize calibration over point accuracy.
\par
Several recent works have sought to address these limitations. 
Hu et al.~\cite{cost} propose a neural architecture that outputs a full predictive density in a single forward pass, substantially reducing computation for deep probabilistic models. 
Zhang et al.~\cite{highDimension} develop an \emph{Improved Deep Mixture Density Network} (IDMDN) for regional wind-power probabilistic forecasting across multiple wind farms, demonstrating robust accuracy in high-dimensional settings. 
R\"ugamer et al.~\cite{flexibility} blend classical structured statistical effects with deep neural networks via semi-structured distributional regression, enabling flexible modeling that accommodates both tabular and image data. 
Finally, Vicario et al.~\cite{explanability} present an uncertainty-aware deep-learning pipeline that assigns reliability scores to predictions based on quantified uncertainty, enhancing interpretability in safety-critical applications. 
Collectively, these advances have helped push forward probabilistic regression and uncertainty estimation.
\par
Recently, Duan et al.~\cite{huan2020} introduced Natural Gradient Boosting (NGBoost), which fits the parameters of a chosen predictive distribution by boosting decision–tree base learners with natural–gradient updates. NGBoost is competitive on many tabular benchmarks with relatively little tuning, making it simple to deploy. However, several limitations arise in regression settings. First, NGBoost requires the user to pre–specify a parametric base distribution (Normal by default), and accuracy can degrade under misspecification. Second, the original formulation is univariate; for multivariate targets one must either train separate models or adopt an extension that models joint uncertainty. While \cite{multivariate} extend NGBoost to multivariate outputs by learning a joint distribution (including covariance structure), this increases computational cost and implementation complexity. Finally, beyond general boosting theory, the original work offers limited task–specific statistical guarantees.
\par 
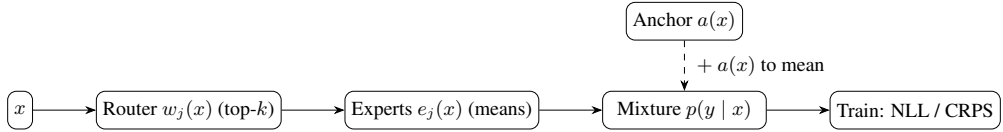
\begin{figure}[t]
\centering
\resizebox{0.80\linewidth}{!}{%
\begin{tikzpicture}[
  >=Stealth, font=\small, node distance=8mm,
  box/.style={draw, rounded corners, inner sep=3pt, minimum height=6mm}
]
  \node[box] (x)      {$x$};
  \node[box, right=10mm of x]   (router) {Router $w_j(x)$ (top-$k$)};
  \node[box, right=10mm of router, minimum width=28mm] (experts) {Experts $e_j(x)$ (means)};
  \node[box, right=10mm of experts, minimum width=26mm] (mix) {Mixture $p(y\mid x)$};
  \node[box, right=10mm of mix] (loss) {Train: NLL / CRPS};

  \node[box, above=8mm of mix] (anchor) {Anchor $a(x)$};
  
  \draw[->] (x) -- (router);
  \draw[->] (router) -- (experts);
  \draw[->] (experts) -- (mix);
  \draw[->] (mix) -- (loss);

  \draw[->, dashed] (anchor) -- node[right, xshift=2pt] {$+\ a(x)$ to mean} (mix);
\end{tikzpicture}%
}
\caption{Modulation of Anchor-MoE with different choices of anchor, router and Mixture-of-Experts.}
\label{fig:architecture}
\end{figure}
\par
To overcome these challenges, we propose \emph{Anchor--MoE}, a simple two-stage, modular architecture for probabilistic and point regression. In Stage~1, a small, tuned gradient-boosted decision trees (GBDT) model furnishes an \emph{anchor mean} $\hat\mu_{\text{GBDT}}(x)$. We $z$-score the target on the training/validation folds, append the standardized anchor as an extra feature, and in Stage~2 map the augmented features (optionally) through a linear projection to a compact latent space. A learnable metric--window kernel with trainable centers and scales scores locality; combined with a soft router, this yields sparse mixture weights over $K$ lightweight mixture-density-network (MDN) experts. Each expert predicts a small Gaussian mixture for the $z$-scored target; in the default \texttt{anchor+delta} mode, expert means provide a heteroscedastic correction to the anchor and the MDN also outputs variances. The mixture is aggregated with the learned weights to form the predictive density. Training minimizes negative log-likelihood of the $z$-scored targets with mild entropy and load-balancing regularization. To improve point accuracy without leakage, a held-out calibration split fits a post-hoc linear map on predicted means ($\mu' = a\,\mu + b$), which is then applied to the test fold; we report RMSE on calibrated means and NLL on the uncalibrated $z$-space density. The design is plug-and-play: the router can be ablated, the anchor can be removed (pure MoE), and MDN heads extend to multivariate outputs by emitting vector means and diagonal or low-rank covariances, leaving the rest of the pipeline unchanged. The architecture of Anchor-MoE is illustrated in Figure~\ref{fig:architecture}

\begin{figure}[ht]
 \centering
\includegraphics[width=0.9\columnwidth]{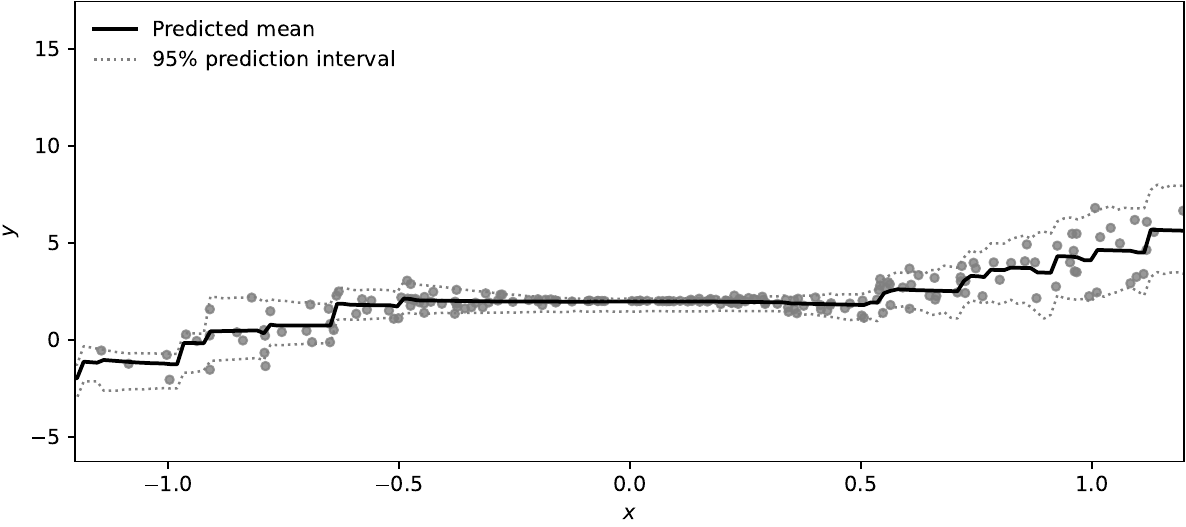}
  \caption{Interval predicted by Anchor-MoE on 1-dimensional toy probabilistic regression problem. Dots represents the data points. Black line is predicted mean and gray lines are upper and lower 95\% covered distribution predicted.}
  \label{fig:1dToy}
\end{figure}
\section*{2. Summary of Contributions}
\par
\begin{enumerate}
  \item \textbf{Anchor--MoE: a simple, modular probabilistic regressor.}
        We propose \emph{Anchor--MoE}, a two-stage architecture. A tuned gradient-boosted trees (GBDT) model furnishes an \emph{anchor mean}; we $z$-score targets, append the standardized anchor as a feature, and (optionally) project inputs to a compact latent space. A learnable metric--window kernel and a soft top-$k$ router yield sparse mixture weights over lightweight mixture-density-network (MDN) experts, which output a heteroscedastic correction (to the anchor) and variances. Training minimizes NLL on $z$-scored targets; a held-out \emph{calibration} split fits a post-hoc linear map on predicted means to improve RMSE without leakage. The design supports multivariate outputs by adapting MDN heads, and admits clean ablations (no-anchor / no-router / no-calibration).

  \item \textbf{Theoretical guarantees under mild assumptions.}
        \begin{itemize}
          \item \emph{Approximation and risk.} Under a H\"older-smooth regression function of order~$\alpha$ and fixed Lipschitz partition-of-unity weights with bounded overlap, Anchor--MoE attains the minimax-optimal $L^2$ risk rate
          $\mathcal{O}\!\big(N^{-2\alpha/(2\alpha+d)}\big)$ (up to constants).
          \item \emph{Generalization.} Using Rademacher-complexity arguments for CRPS, the population--empirical gap scales as
          $\widetilde{\mathcal O}\!\Big(\sqrt{(\log(Mh)+P+K)/N}\Big)$; under uniformly bounded means and variances, the same
          $\widetilde{\mathcal O}\!\big(\sqrt{(\log(Mh)+P+K)/N}\big)$ scaling holds for test NLL up to constants.
          Here $M$ is the number of Gaussian components per expert, $h$ the expert width, $P$ the router size (which subsumes any latent dimension), and $K$ the number of experts; under bounded-overlap (top-$k$) routing, $K$ can be replaced by $k$.
          \item \emph{High-dimensional scaling.} If data lie on a $d_0$-dimensional $C^1$ manifold (or depend on $s\ll d$ coordinates), the rate adapts to the intrinsic dimension, i.e., $N^{-2\alpha/(2\alpha+d_0)}$ (or $N^{-2\alpha/(2\alpha+s)}$).
        \end{itemize}
\end{enumerate}

\section*{3. Anchor--MoE}

In standard prediction settings the object of interest is a scalar functional such as
\(\mathbb{E}[Y\mid X=x]\).  
In probabilistic regression we instead aim to learn a full predictive law
\(P_{\Theta(x)}(y\mid x)\).
Our approach is to parameterize \(P_{\Theta(x)}\) by a \emph{locally simple, globally flexible}
mixture family whose parameters \(\Theta(x)\) are smooth functions of the input.
\par
Concretely, Anchor--MoE first forms a strong \emph{anchor mean} \(\mu_{\mathrm{a}}(x)\) using a
small gradient–boosted tree.  
The anchor is concatenated to the features and mapped to a compact latent space, from which a
metric–window router produces \emph{sparse} (soft top-\(k\)) mixture weights.  
Each activated expert is a lightweight MDN that predicts a local residual (“delta”) to the anchor
and a scale, so that the resulting predictive distribution is a mixture with means
\(\mu_{\mathrm{a}}(x)+\delta\) and heteroscedastic variances.  
\par
The next subsections detail the components: The latent projection and metric window (Sec.~3.1), the
latent metric–window and router (Sec.~3.2), the expert MDN heads and training objective
(Sec.~3.3), and the post-hoc mean calibration (Sec.~3.4).
\par
\paragraph{Notation.}
Throughout Sec.~3 we write $x\in\mathbb{R}^d$ for the (optionally anchor–augmented) input and $y$ for the z-scored target.
Let $\mu_{\mathrm{anc}}(x)$ denote the anchor mean (from GBDT), and
$z=\operatorname{LayerNorm}(W_\phi x+b_\phi)\in\mathbb{R}^D$ the latent code. The anchor is initialized as GBDT by defaut.
The metric–window has centers $c_j\in\mathbb{R}^D$ and diagonal scales $s_j\in\mathbb{R}^{D}_{+}$, producing unnormalised weights
$\tilde w_j(z)=\exp\!\big(-\tfrac12\,(z-c_j)^\top \mathrm{diag}(s_j)^{-2}(z-c_j)\big)$ and normalised
$w_j(z)=\tilde w_j(z)\big/\sum_{i=1}^K \tilde w_i(z)$.
The router forms a query $q=W_q z$ and keys $k_j$, with $g_j(z)=\mathrm{softmax}(\langle q,k_j\rangle/(\sqrt{d_r}\,\tau))$.
Combined (and top-$k$ masked) gates are $\alpha_j(z)\propto w_j(z)\,g_j(z)$ with $\sum_{j=1}^K \alpha_j(z)=1$.
Expert $j$ (an MDN) outputs mixture weights $\pi_{j,c}(x)$, means $\mu_{j,c}(x)$, and scales $\sigma_{j,c}(x)>0$ for $c=1,\dots,C$.
In \emph{anchor+delta} mode the effective mean is
$\mu^{\mathrm{eff}}_{j,c}(x)=\mu_{\mathrm{anc}}(x)+\Delta_{j,c}(x)$ (otherwise $\mu^{\mathrm{eff}}_{j,c}(x)=\mu_{j,c}(x)$).
The predictive density is
\[
p(y\mid x)=\sum_{j=1}^{K}\alpha_j(z)\sum_{c=1}^{C}\pi_{j,c}(x)\,
\mathcal N\!\big(y;\,\mu^{\mathrm{eff}}_{j,c}(x),\,\sigma_{j,c}^{2}(x)\big).
\]
Unless noted, losses are computed in z-space; RMSE is reported after inverting the z-score, while NLL is reported on the z-space density.
\par
Throughout we instantiate the anchor with a gradient-boosted trees (GBDT) model for simplicity, but any fixed point predictor $a(x)$ trained on the TR/VA split can be used in its place. The anchor value is concatenated as a feature and, in the anchor+$\Delta$ mode, experts learn a residual on top of $a(x)$.

\begin{figure}[ht]
\centering\includegraphics[width=0.9\columnwidth]{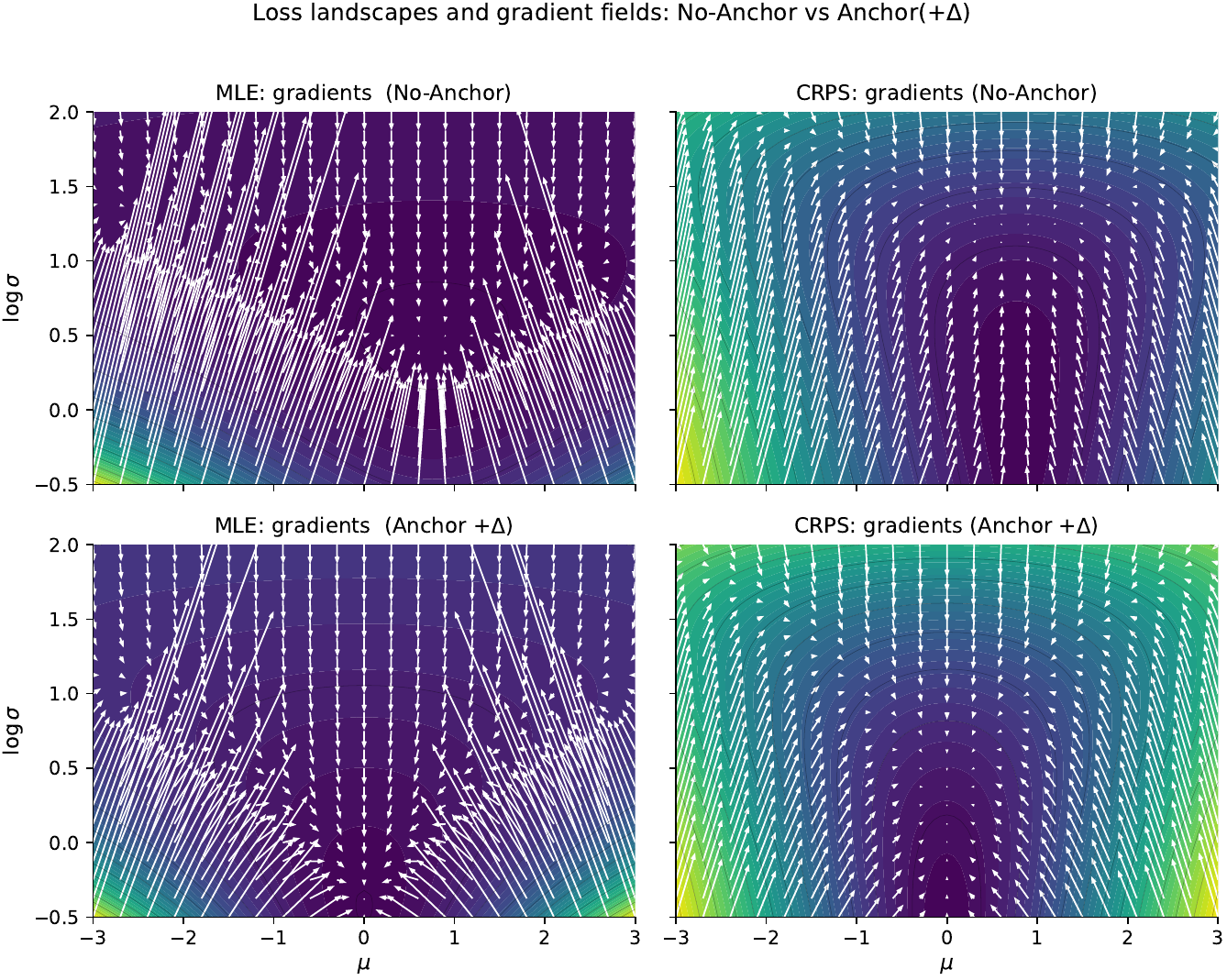}
\caption{Proper scoring rules and corresponding gradients for fitting a normal distribution $ \sim \mathcal{N}(0,1)$. For each scoring rule, the landscape of the score (colors and contours) is identical.}
\label{fig:gradientField}
\end{figure}

\subsection*{3.1 Latent Projection and Metric Window}
We map the (optionally anchor–augmented) input $x\in\mathbb{R}^d$ to a compact latent code
\[
  z \;=\; \operatorname{LayerNorm}\!\bigl(W_\phi x + b_\phi\bigr)\in\mathbb{R}^D .
\]
Locality is scored by a learnable metric–window kernel with centers $\{c_j\}_{j=1}^K$ and
diagonal scale vectors $s_j\in\mathbb{R}_{+}^{D}$ (with $S_j:=\mathrm{diag}(s_j)$):
\[
  \tilde w_j(z)
  \;=\;
  \exp\!\Bigl(-\tfrac12\,(z-c_j)^\top S_j^{-2}(z-c_j)\Bigr),
  \qquad
  w_j(z)
  \;=\;
  \frac{\tilde w_j(z)}{\sum_{i=1}^K \tilde w_i(z)}
  \;\in\; \Delta^{K-1}.
\]
For numerical stability we constrain the log–scales componentwise to
$\log s_{j,\ell}\in[\underline{\tau},\overline{\tau}]$ and add a mild scale regularizer
$\lambda_s\sum_{j=1}^K\| \log s_j\|_2^2$ to the loss.

To obtain sparse and robust gates we apply a smoothed top-$k$ mask to $w(z)$ at both training
and inference. Let $m_k(z)\in\{0,1\}^K$ be the indicator of the $k$ largest entries of $w(z)$.
With a small $\varepsilon\in(0,1)$ the final gates are
\[
  \bar w(z)
  \;=\;
  (1-\varepsilon)\,
  \frac{w(z)\odot m_k(z)}{\mathbf{1}^\top\!\bigl(w(z)\odot m_k(z)\bigr)}
  \;+\;
  \frac{\varepsilon}{k}\, m_k(z),
\]
which preserves probability mass on the top-$k$ experts while preventing zero gradients within
the active set.

\subsection*{3.2 Router}
Given the latent code $z\in\mathbb{R}^D$, we form a query
$q=W_q z\in\mathbb{R}^{d_r}$ and maintain keys
$\{k_j\}_{j=1}^K\subset\mathbb{R}^{d_r}$. We use scaled dot-product
logits (optionally cosine–normalised) with temperature $\tau>0$:
\[
  \tilde\ell_j(z)\;=\;\frac{\langle q,\,k_j\rangle}{\sqrt{d_r}}\,,\qquad
  \ell_j(z)\;=\;\tilde\ell_j(z)/\tau \quad (j=1,\dots,K).
\]
Let $w_j(z)$ be the locality weights from Section~3.1.
Instead of multiplying and re-normalising, we combine them in log-space
for numerical stability:
\[
  \alpha(z)
  \;=\;
  \mathrm{softmax}\!\Big(\,\log\!\big(w(z)\vee \epsilon_w\big)\;+\;\ell(z)\,\Big),
\]
where $(a\vee b)=\max\{a,b\}$ is applied elementwise and
$\epsilon_w\!>\!0$ is a tiny floor (e.g.\ $10^{-12}$). This is exactly
equivalent to $\alpha_j \propto w_j\,\mathrm{softmax}(\ell)_j$ but is
more stable when some $w_j$ are very small.

For specialisation and robustness we keep only the top-$k$ entries of
$\alpha$ (others set to $0$) and apply a light smoothing within the
active set (as in Section~3.1). With $m_k(z)\in\{0,1\}^K$ the
top-$k$ indicator and $\varepsilon\in(0,1)$,
\[
  \bar\alpha(z)
  \;=\;
  (1-\varepsilon)\,
  \frac{\alpha(z)\odot m_k(z)}
       {\boldsymbol{1}^\top\!\bigl(\alpha(z)\odot m_k(z)\bigr)+\varepsilon_{\mathrm{stab}}}
  \;+\;
  \frac{\varepsilon}{k}\, m_k(z),
\]
where $\varepsilon_{\mathrm{stab}}\approx 10^{-12}$.
In practice $m_k$ is treated as a non-differentiable mask (stop-gradient).

This router adds $\mathcal{O}(K d_r)$ work per example, captures
content-dependent gating, and complements the locality kernel by
suppressing far-away experts.

\subsection*{3.3 Mixture of Experts}
Given the final locality--aware mixture weights $\bar\alpha(z)\in\Delta^{K-1}$
(from the metric--window kernel and the lightweight router with top-$k$ gating;
Sec.~3.1--3.2), we dispatch each input to $K$ specialised density estimators
(``experts''). Expert $j$ is a small Mixture Density Network (MDN) that outputs,
for $c=1,\dots,C$, a component weight $\pi_{j,c}(x)$, a mean head, and a scale
head. We enforce $\sum_{c}\pi_{j,c}(x)=1$ via a softmax over logits, and predict
a log--scale $t_{j,c}(x)$ with
\[
  \sigma_{j,c}(x)\;=\;\bigl[\exp\{t_{j,c}(x)\}\bigr]_{\sigma_{\min}}^{\sigma_{\max}}
  \quad\text{(componentwise clamp)},\qquad
  0<\sigma_{\min}\le\sigma_{\max}<\infty .
\]

\paragraph{Anchor coupling}
Let $\mu_{\text{anc}}(x)$ denote the (z-scored) anchor mean produced by the
gradient-boosted baseline. The expert mean heads either predict an absolute
mean or a residual $\Delta_{j,c}(x)$ added to the anchor:
\[
  \mu^{\text{eff}}_{j,c}(x) \;=\;
  \begin{cases}
    \mu_{\text{anc}}(x) + \Delta_{j,c}(x), & \text{(anchor+delta, default)}\\[2pt]
    \mu_{\text{anc}}(x),                   & \text{(anchor only)}\\[2pt]
    \mu_{j,c}(x),                          & \text{(free, no anchor)}
  \end{cases}
\]
with a mild $\ell_2$ penalty $\lambda_\Delta\,\E[\Delta_{j,c}(X)^2]$ to discourage
unnecessary drift in the anchor+delta mode. When enabled, $\mu_{\text{anc}}(x)$ is
also concatenated to the expert/router inputs.

\paragraph{Predictive density}
For a univariate target (our default experiments), the conditional density is
\[
  p(y\mid x)
  \;=\;
  \sum_{j=1}^{K} \bar\alpha_j(z)\;
  \sum_{c=1}^{C} \pi_{j,c}(x)\;
  \mathcal{N}\!\bigl(y;\,\mu^{\text{eff}}_{j,c}(x),\,\sigma^{2}_{j,c}(x)\bigr).
\]
This design lets experts specialise on local regions (capturing heteroscedasticity
and multi-modality), while the gated mixture yields smooth probabilistic
interpolation across regions. The formulation extends to multivariate $y$ by
emitting vector means and a diagonal/low-rank (or full) covariance; we focus on
the univariate case here for clarity.
\subsection*{3.4 Calibration}

We use a small, disjoint \emph{calibration split} to correct the systematic bias of the predicted mean without leaking test information.
Let $\mathrm{Cal}$ index the calibration set (targets are z-scored as in training) and denote
$\hat\mu_i:=\hat\mu(x_i)$ the model's predicted mean in z-space.
We fit an affine map $(a,b)$ by least squares:
\[
  (a^\star,b^\star)
  \;=\;
  \arg\min_{a,b}\ \frac{1}{|\mathrm{Cal}|}\sum_{i\in\mathrm{Cal}}
  \bigl(y_i - (a\,\hat\mu_i + b)\bigr)^2,
\]
which has the closed form
\[
  a^\star \;=\; \frac{\mathrm{Cov}_{\mathrm{Cal}}(\hat\mu, y)}{\mathrm{Var}_{\mathrm{Cal}}(\hat\mu)},
  \qquad
  b^\star \;=\; \overline{y}_{\mathrm{Cal}} - a^\star\,\overline{\hat\mu}_{\mathrm{Cal}}.
\]
At test time we output the \emph{calibrated} mean
$\hat\mu_{\mathrm{cal}}(x)=a^\star\,\hat\mu(x)+b^\star$ in z-space and report RMSE after inverting the z-score to original units.
Unless otherwise stated, we keep the predictive variance unchanged and report NLL on the original (uncalibrated) z-space density to avoid optimistic bias.

\begin{algorithm}[ht]
\caption{Anchor--MoE training, calibration, and testing}
\label{alg:anchor-moe}
\begin{algorithmic}[1]

\State \textbf{Split:}
\Statex \quad $\mathcal D \to \mathcal D_{\text{train}} \,\dot\cup\, \mathcal D_{\text{test}}$;\quad
$\mathcal D_{\text{train}} \to \mathcal D_{\text{TV}} \,\dot\cup\, \mathcal D_{\text{cal}}$;\quad
$\mathcal D_{\text{TV}} \to \mathcal D_{\text{tr}} \,\dot\cup\, \mathcal D_{\text{va}}$.

\Statex
\State \textbf{GBDT selection (on TR/VA):}
\For{$t=1,\dots,T_g$}
  \State $e_t \gets \mathrm{RMSE}\!\big(y_{\text{va}},\, \mathrm{GBDT}_t(X_{\text{va}})\big)$
\EndFor
\State $t^\star \gets \arg\min_t e_t$
\State \textit{Train} a fresh $\mathrm{GBDT}_{t^\star}$ on $(X_{\text{tr}},y_{\text{tr}})$ to obtain $f_{\text{sub}}$
\State \textit{Refit} $\mathrm{GBDT}_{t^\star}$ on $(X_{\text{TV}},y_{\text{TV}})$ to obtain $\hat f$

\Statex
\State \textbf{Phase-1 (TR/VA): anchor z-score, feature standardization, MoE early selection}
\State $(\mu_{\text{tr}},\sigma_{\text{tr}}) \gets \mathrm{mean/std}(y_{\text{tr}})$
\State $z_{\text{tr}} \gets \mathrm{zsc}(y_{\text{tr}};\mu_{\text{tr}},\sigma_{\text{tr}})$;\quad
       $z_{\text{va}} \gets \mathrm{zsc}(y_{\text{va}};\mu_{\text{tr}},\sigma_{\text{tr}})$
\State $\alpha_{\text{tr}} \gets \mathrm{zsc}\!\big(f_{\text{sub}}(X_{\text{tr}});\mu_{\text{tr}},\sigma_{\text{tr}}\big)$;\quad
       $\alpha_{\text{va}} \gets \mathrm{zsc}\!\big(f_{\text{sub}}(X_{\text{va}});\mu_{\text{tr}},\sigma_{\text{tr}}\big)$
\State $\tilde X_{\text{tr}} \gets [X_{\text{tr}},\,\alpha_{\text{tr}}]$;\quad
       $\tilde X_{\text{va}} \gets [X_{\text{va}},\,\alpha_{\text{va}}]$
\State $(m_{\text{tr}},s_{\text{tr}}) \gets \mathrm{col\text{-}mean/std}(\tilde X_{\text{tr}})$
\State $\bar X_{\text{tr}} \gets \mathrm{std}(\tilde X_{\text{tr}};m_{\text{tr}},s_{\text{tr}})$;\quad
       $\bar X_{\text{va}} \gets \mathrm{std}(\tilde X_{\text{va}};m_{\text{tr}},s_{\text{tr}})$
\State initialize $\Theta_1$
\For{$t=1,\dots,T_{\max}$}
  \State $\Theta_{t+1} \gets \Theta_t - \eta\,\nabla_{\Theta}\,\mathrm{NLL}\!\big(\bar X_{\text{tr}},\, z_{\text{tr}};\,\Theta_t\big)$
\EndFor
\State $t^\star_{\text{MoE}} \gets \arg\min_t \mathrm{NLL}\!\big(\bar X_{\text{va}},\, z_{\text{va}};\,\Theta_t\big)$
\State $\Theta^\dagger \gets \Theta_{t^\star_{\text{MoE}}}$

\Statex
\State \textbf{Phase-2 (TV/CAL/TEST): freeze early epoch, refit on TV, prep CAL/TEST}
\State $(\mu_{\text{tv}},\sigma_{\text{tv}}) \gets \mathrm{mean/std}(y_{\text{TV}})$
\State $z_{\text{tv}} \gets \mathrm{zsc}(y_{\text{TV}};\mu_{\text{tv}},\sigma_{\text{tv}})$
\For{$S \in \{\text{TV},\,\text{cal},\,\text{test}\}$}
  \State $\alpha_S \gets \mathrm{zsc}\!\big(\hat f(X_S);\mu_{\text{tv}},\sigma_{\text{tv}}\big)$;\quad
         $\tilde X_S \gets [X_S,\,\alpha_S]$
\EndFor
\State $(m_{\text{tv}},s_{\text{tv}}) \gets \mathrm{col\text{-}mean/std}(\tilde X_{\text{TV}})$
\State $\bar X_S \gets \mathrm{std}(\tilde X_S;m_{\text{tv}},s_{\text{tv}})$ for $S\in\{\text{TV},\text{cal},\text{test}\}$
\State reload $\Theta^\dagger$
\For{$t=1,\dots, t^\star_{\text{MoE}}$}
  \State $\Theta \gets \Theta - \eta\,\nabla_{\Theta}\,\mathrm{NLL}\!\big(\bar X_{\text{TV}},\, z_{\text{tv}};\,\Theta\big)$
\EndFor

\Statex
\State \textbf{Calibration (on CAL): linear post-hoc map for mean)}
\State $\hat\mu_{\text{cal}}^{\text{orig}} \gets \sigma_{\text{tv}}\cdot \hat\mu_z(\bar X_{\text{cal}};\Theta) + \mu_{\text{tv}}$
\State $(a,b) \gets \arg\min_{a,b}\,\big\|\,a\,\hat\mu_{\text{cal}}^{\text{orig}} + b - y_{\text{cal}}\,\big\|_2^2$

\Statex
\State \textbf{Test: report calibrated RMSE (orig) and NLL (z-space)}
\State $\hat\mu_{\text{test}}^{\text{orig}} \gets \sigma_{\text{tv}}\cdot \hat\mu_z(\bar X_{\text{test}};\Theta) + \mu_{\text{tv}}$
\State $\hat\mu_{\text{test}}^{\text{cal}} \gets a\,\hat\mu_{\text{test}}^{\text{orig}} + b$
\State $\mathrm{RMSE} \gets \mathrm{RMSE}\!\big(y_{\text{test}},\,\hat\mu_{\text{test}}^{\text{cal}}\big)$
\State $\mathrm{NLL}_z \gets \mathrm{NLL}\!\big(\bar X_{\text{test}},\,\mathrm{zsc}(y_{\text{test}};\mu_{\text{tv}},\sigma_{\text{tv}});\Theta\big)$
\State \Return $\Theta^\ast\!=\!\Theta,\ (a,b),\ \mathrm{RMSE},\ \mathrm{NLL}_z$

\end{algorithmic}
\end{algorithm}

\section*{4. Theoretical Analysis}\label{sec:theory}

\noindent\textbf{Why theory and how to use it.}
The theory serves two purposes. First, it certifies that Anchor--MoE is statistically competitive: under mild smoothness and bounded–overlap assumptions it achieves minimax–optimal $L^2$ rates, adapts to intrinsic dimension (manifolds or sparsity), and enjoys $\widetilde{\mathcal O}(N^{-1/2})$ generalization gaps for proper scoring rules. Second, it provides \emph{practical guidance} for model sizing: the bias--variance decomposition yields an approximation term $C_1K^{-2\alpha/d}$ and an estimation term $C_2(k\,\mathrm{comp})K/N$, suggesting
\[
K \asymp N^{d/(2\alpha+d)}\quad
\]
Because bounded-overlap/top-$k$ routing makes the estimation term depend on the constant $k$ rather than $K$ per input, increasing $K$ beyond this balance rarely helps unless $N$ grows. The CRPS/NLL generalization bound scales like $\widetilde{\mathcal O}\!\big(\sqrt{(\log(Mh)+P+K)/N}\big)$, so keep the mixture size $M$ small (e.g., $2$--$5$), hidden width $h$ moderate, and router size $P$ controlled. Finally, the NLL–$L^2$ link justifies training with Gaussian NLL while reporting mean-squared risk for the predictive mean.

To be more specific, we analyze Anchor--MoE as a probabilistic regressor but evaluate risk on the
\emph{predictive mean}
\[
  \widehat f(x)\;:=\;\mathbb{E}_{\widehat p(\cdot\mid x)}[\,Y\,] .
\]
Unless stated otherwise, $L^2$ norms are taken with respect to the Lebesgue measure on $[0,1]^d$;
on a $d_0$-dimensional manifold $\mathcal{M}$ we use the normalized volume measure $\mu_{\mathcal M}$
(Hausdorff) for $L^2(\mathcal M)$.
\par
Gating is implemented by a \emph{fixed (non-learned)} compactly supported partition of unity (PoU)
with bounded overlap $k$, so at each $x$ at most $k$ experts have nonzero weight.
Each expert mean belongs to a class of \emph{bounded capacity} whose width/depth/covering numbers
do not scale with $K$. These are the standing assumptions used in the appendix
(Assumptions~A1--A3 and, for generalization under CRPS/NLL, G1--G2).
We write $\Rad_N(\cdot)$ for empirical Rademacher complexity. The pseudocode is demonstrated in Algorithm~\ref{alg:anchor-moe}.

\subsection*{4.1 Approximation and minimax-optimal rates}\label{subsec:approx}

Let $f^\star\!\in\!\cF_\alpha(L)$ be $\alpha$-H\"older on $[0,1]^d$.
With $K$ PoU windows of mesh $h\asymp K^{-1/d}$ (fixed, non-learned, bounded overlap $k$), the local interpolant has squared error $\mathcal{O}(K^{-2\alpha/d})$.
Estimating $K$ expert means with overlap $k$ and per–expert capacity $\mathrm{comp}$ contributes $\mathcal{O}(k\,\mathrm{comp}\,K/N)$.
Balancing the two terms yields the classical minimax rate:

\begin{quote}
\textbf{Main bound (see Appendix~A1.5, Thm.~\ref{thm:main_safe}).}
Under the assumptions above,
\[
  \mathbb{E}\!\bigl[\|\widehat f-f^\star\|_{L^2}^2\bigr]
  \;\le\;
  C_1\,K^{-2\alpha/d}
  \;+\;
  C_2\,\frac{k\,\mathrm{comp}\,K}{N}.
\]
Choosing $K^\star\asymp N^{d/(2\alpha+d)}$ gives
\[
  \sup_{f^\star\in\cF_\alpha(L)}
  \mathbb{E}\!\left[\|\widehat f-f^\star\|_{L^2}^2\right]
  \;\lesssim\;
  N^{-2\alpha/(2\alpha+d)},
\]
matching the information-theoretic lower bound.
\end{quote}

We train with Gaussian NLL while evaluating $L^2$ risk on the predictive mean.
This is justified by an \emph{NLL--$L^2$ link}: when variances are bounded away from $0$ and $\infty$, the excess NLL controls the mean-squared error up to constants, plus a variance-mismatch term.

\subsection*{4.2 Generalisation under CRPS}\label{subsec:gen}

For probabilistic calibration we assess CRPS.
Using that CRPS is $2$-Lipschitz in the predictive CDF under the $L^1$ metric and that the loss is bounded once expert means and variances are bounded (e.g., $|e_j(x)|\!\le\!R_f$, $\sigma(x)\!\in[\underline\sigma,\overline\sigma]$, and $y\!\in[-R_y,R_y]$, giving $B\!\le\!R_f+R_y+\sqrt{2/\pi}\,\overline\sigma$), the population–empirical gap admits a standard Rademacher bound (see Appendix~A2, Thm.~\ref{thm:gen}):
\begin{quote}
\textbf{Generalisation bound.}
For any $\delta\in(0,1)$, with probability at least $1-\delta$,
\[
\mathcal R(\hat\theta,\hat\phi)-\widehat{\mathcal R}_N(\hat\theta,\hat\phi)
\;\le\;
4\,\Rad_N(\Fclass)
\;+\;
3B\,\sqrt{\frac{\log(2/\delta)}{2N}}.
\]
Under the capacity assumptions (Appendix~A2) we have
$\Rad_N(\Fclass)\ \le\ C_{\!*}\sqrt{(\log(Mh)+P+K)/N}$, hence the gap scales as
$\widetilde{\mathcal O}\!\big(\sqrt{(\log(Mh)+P+K)/N}\big)$.
Here $M$ is the number of mixture components per expert, $h$ the expert width proxy, $P$ the router size, and $K$ the number of experts.
\end{quote}

With top-$k$ bounded-overlap gating, the dependence on $K$ inside $\Rad_N(\Fclass)$ can be replaced by the constant $k$.

\subsection*{4.3 High-dimensional scaling}\label{subsec:high-d}

Anchor--MoE adapts to intrinsic dimension in two ubiquitous regimes.

\paragraph{(a) Data on a manifold.}
Let \(X\) lie on a compact \(C^1\) submanifold \(\mathcal M\subset[0,1]^d\) of intrinsic dimension \(d_0\) (positive reach),
and measure risk in \(L^2(\mathcal M)\) w.r.t.\ the normalised Hausdorff measure.
Using a fixed geodesic PoU with bounded overlap, the approximation term becomes \(K^{-2\alpha/d_0}\)
while the estimation term is unchanged:

\begin{quote}
\textbf{Manifold rate (see Appendix Theorem~\ref{thm:manifold-rate}).}
\(
\displaystyle
\mathbb{E}\bigl\|\widehat f-f^\star\bigr\|_{L^2(\mathcal M)}^2
\;\le\;
C_1 K^{-2\alpha/d_0}
\;+\;
C_2\,\frac{k\,\mathrm{comp}\,K}{N},
\)
so \(K^\star\asymp N^{d_0/(2\alpha+d_0)}\) yields the rate
\(N^{-2\alpha/(2\alpha+d_0)}\).
\end{quote}
\paragraph{(b) Sparse coordinate dependence.}
If \(f^\star(x)=g^\star(x_S)\) depends on only \(s\ll d\) coordinates \(S\) (oracle known),
a PoU on the \(s\)-dimensional subspace gives:
\begin{quote}
\textbf{Sparse rate (see Appendix Theorem~\ref{thm:sparse-rate}).}
\(
\displaystyle
\mathbb{E}\bigl\|\widehat f-f^\star\bigr\|_{L^2}^2
\;\le\;
C_1 K^{-2\alpha/s}
\;+\;
C_2\,\frac{k\,\mathrm{comp}\,K}{N},
\)
so \(K^\star\asymp N^{s/(2\alpha+s)}\) yields
\(N^{-2\alpha/(2\alpha+s)}\).
If \(S\) must be learned, a model-selection penalty
\(\,\widetilde{\mathcal O}\!\bigl((s\log d)/N\bigr)\) typically augments the estimation term.
\end{quote}

Here \(L^2\) is w.r.t.\ the marginal law of \(X\) (on \(\mathcal M\) in (a), on \([0,1]^s\) in (b)); if the marginal density of \(X_S\) is bounded above/below, constants depend only on these bounds.
\par
Choose \(K\) by balancing
\(K^{-2\alpha/\text{(intrinsic dim)}} \approx (k\,\mathrm{comp}\,K)/N\):
\(K\propto N^{d/(2\alpha+d)}\) in full dimension,
\(K\propto N^{d_0/(2\alpha+d_0)}\) on manifolds,
and \(K\propto N^{s/(2\alpha+s)}\) under sparsity.

\section*{5. Experiments}
\paragraph{Theory-guided setup.}
We choose hyperparameters to match the assumptions used in our analysis: a fixed metric–window with top-$k$ routing (bounded overlap $k$), experts of fixed capacity (width/depth not growing with $K$), and variance clamping $\sigma(x)\!\in\![\sigma_{\min},\sigma_{\max}]$ (Assump.~G1).
We keep $K$ small and constant across datasets (here $K{=}8$), consistent with balancing the $\mathcal{O}(K^{-2\alpha/d})$ approximation term and the $\mathcal{O}(k\,\mathrm{comp}\,K/N)$ estimation term. 
We report NLL in $z$-space and RMSE on post-hoc calibrated means as in Sec.~3.4; the CRPS-based generalization bound in Appendix~A2 yields the same $\widetilde{\mathcal O}(N^{-1/2})$ rate for bounded-NLL, up to constants.
\par
Each experiment follows the protocol of \cite{hernandez2015probabilistic} on nine UCI benchmark datasets.
For each dataset we perform 20 random outer splits with 90\% for training and 10\% for testing.
Within each outer training fold, we hold out 20\% as a validation set to select the number of GBDT boosting stages
that \emph{maximizes} validation log-likelihood (equivalently, minimizes NLL).
The chosen stage is then refit on the full 90\% training fold before training the MoE.
For the \textsc{Protein} dataset, to speed up training we randomly subsample 10{,}000 examples \emph{per run} prior to the 90/10 split.
We repeat the entire process 20 times and report the mean \(\pm\) standard error across runs.
Unless otherwise stated, RMSE is computed on post-hoc calibrated means on the original scale, while NLL is computed on the
un-calibrated predictive density in the \(z\)-scored space.
\par
For all experiments we use a common configuration. The latent projection dimension is \(D=2\). We set the number of experts to \(K=8\) and use top-\(k\) gating with \(k=2\) (with light smoothing). Each expert is a small MLP with hidden width \(128\) and a 3-component MDN head (\(C=3\)). We train with a learning rate of \(10^{-3}\) for \(400\) epochs, and clamp predicted standard deviations to the interval \([5\times 10^{-2},\,1]\). The simulation with toy daset is illustrated in Figure~\ref{fig:1dToy} and Figure~\ref{fig:dynamics}. Figure~\ref{fig:gradientField} shows how anchor part will improve the parameter updating.

For \textsc{Protein}, which is subsampled to 10{,}000 examples per run to speed up training, we re-train NGBoost on the same subsamples for a fair comparison. For the other datasets we use 100\% of the data and report the NGBoost results from \cite{huan2020}. Results are shown in Table~\ref{tab:uci_nll} and Table~\ref{tab:uci_rmse} for uncertainty and point estimation respectively.
\par
We evaluate uncertainty quality using the average test negative log-likelihood (NLL):
\[
\mathrm{NLL}\;=\;-\frac{1}{|\mathcal{D}_{\mathrm{test}}|}\sum_{(x,y)\in\mathcal{D}_{\mathrm{test}}}
\log p_{\theta}(y\mid x),
\]
where lower is better. Our primary comparison is against NGBoost; results for other baselines are reported in the Appendix.

Although Anchor--MoE is designed for uncertainty estimation, a point prediction is readily obtained as the predictive mean \(\hat\mu(x)=\mathbb{E}_{p_{\theta}(\cdot\mid x)}[Y]\). We assess point prediction quality with test RMSE:
\[
\mathrm{RMSE}\;=\;\sqrt{\frac{1}{|\mathcal{D}_{\mathrm{test}}|}\sum_{(x,y)\in\mathcal{D}_{\mathrm{test}}}\bigl(\hat\mu(x)-y\bigr)^2}.
\]
Unless otherwise noted, we use the same configuration as in the uncertainty experiments (and apply the post-hoc mean calibration for RMSE as described in Section~3.4). Ablation results are shown in Table~\ref{tab:uci_ablation_nll} and in Table~\ref{tab:uci_rmse} for uncertainty and point estimation respectively.We use GBDT as the default anchor for reproducibility; other anchors can be dropped in without changing the training pipeline.
\par
\paragraph{Ablation protocol.}
To quantify the contribution of each key component, we conduct ablation studies under the same configuration as in the uncertainty experiments. We report both test NLL (on the uncalibrated $z$-space density) and test RMSE (on the predictive mean).

\begin{itemize}[leftmargin=1.5em]
\item \textbf{Anchor.} The default \texttt{anchor+delta} mode uses a small GBDT to produce an anchor mean $\hat\mu_{\text{anc}}(x)$; expert heads predict residuals that correct this anchor and also output variances. In \textbf{No-Anchor}, we remove the anchor feature and residual coupling so the experts predict free means. This isolates the respective contributions of the anchor and the MoE.
\item \textbf{Router.} The router is a lightweight dot-product router $g(z)$ combined multiplicatively with the metric–window weights $w(z)$ and top-$k$ gating. In \textbf{No-Router}, we set $g(z)\!\equiv\!1$ and rely solely on the metric–window (with the same top-$k$ mask and smoothing), then renormalize. This tests the benefit of content-dependent routing.
\item \textbf{Calibration.} A 1D linear calibrator $\mu' = a\,\mu + b$ is fit on a held-out split to improve point accuracy. In \textbf{No-Cal}, we report RMSE using the uncalibrated means.
\end{itemize}

\begin{table}[H]
  \centering
  \caption{Comparison on UCI Benchmark dataset as measured by NLL. Results for NGBoost are reported from\cite{huan2020}, for other methods results are also reported from thier previous work, see Table~\ref{tab:moe_others} in appendix. Anchor-MoE provide the competitive performance in terms of NLL on complex dataset. Best results are bolded with standard error.}
  \label{tab:uci_nll}
  \begin{tabular}{l c c c}
    \toprule
    Dataset & \textit{N} & Anchor-MoE & NGBoost \\
    \midrule
    Boston  & 506  & \textbf{0.60 $\pm$ 0.11} & 2.43 $\pm$ 0.15\\
    
    Concrete  & 1030 & \textbf{0.25 $\pm$ 0.06} & 3.04 $\pm$ 0.17\\
    
    Energy & 768 & \textbf{-1.68 $\pm$ 0.2} & 0.46 $\pm$ 0.06\\
    
    Kin8nm & 8192&0.12 $\pm$ 0.01  & \textbf{-0.49 $\pm$ 0.02}\\
    
    Naval & 11934& -1.26 $\pm$ 0.02&  \textbf{-5.34 $\pm$ 0.04}  \\
    
    Power & 9568 & \textbf{-0.15 $\pm$ 0.02} & 2.79 $\pm$ 0.11 \\
    
    Protein & 10000 & \textbf{1.06$\pm$ 0.04} & 1.24 $\pm$ 0.04 \\
    
    Wine & 1599 & \textbf{1.20$\pm$ 0.02}  & 4.96 $\pm$ 0.60 \\
    
    Yacht & 308 &  \textbf{-1.80$\pm$ 0.04}  &  0.2 $\pm$ 0.26  \\
    \bottomrule
  \end{tabular}
\end{table}

\begin{table}[H]
  \centering
  \caption{Comparison on UCI Benchmark dataset as measured by RMSE. Anchor-MoE offers similar results as NGBoost.Bolding is as in Table 1.}
  \label{tab:uci_rmse}
  \begin{tabular}{l c c c}
    \toprule
    Dataset & \textit{N} & Anchor-MoE & NGBoost\\
    \midrule
    Boston  & 506  & 3.01 $\pm$ 0.14 & \textbf{2.94 $\pm$ 0.53} \\
    
    Concrete  & 1030 & \textbf{4.45 $\pm$ 0.16} & 5.06 $\pm$ 0.61  \\
    
    Energy & 768 & 0.47 $\pm$ 0.02 & \textbf{0.46 $\pm$ 0.06}\\
    
    Kin8nm & 8192& \textbf{0.07 $\pm$ 0.00} & 0.16 $\pm$ 0.00\\
    
    Naval & 11934& 0.00 $\pm$ 0.00 &  0.00 $\pm$ 0.00  \\
    
    Power & 9568 & \textbf{3.21 $\pm$ 0.05} & 3.79 $\pm$ 0.18 \\
    
    Protein & 10000 & \textbf{4.41 $\pm$ 0.02} & \textbf{4.44 $\pm$ 0.02} \\
    
    Wine & 1599 & 0.62 $\pm$ 0.01 & \textbf{0.60 $\pm$ 0.01} \\
    
    Yacht & 308 & 0.62 $\pm$ 0.06& \textbf{0.50 $\pm$ 0.20} \\
    \bottomrule
  \end{tabular}
\end{table}

\begin{figure}[h]
\centering\includegraphics[width=0.9\columnwidth]{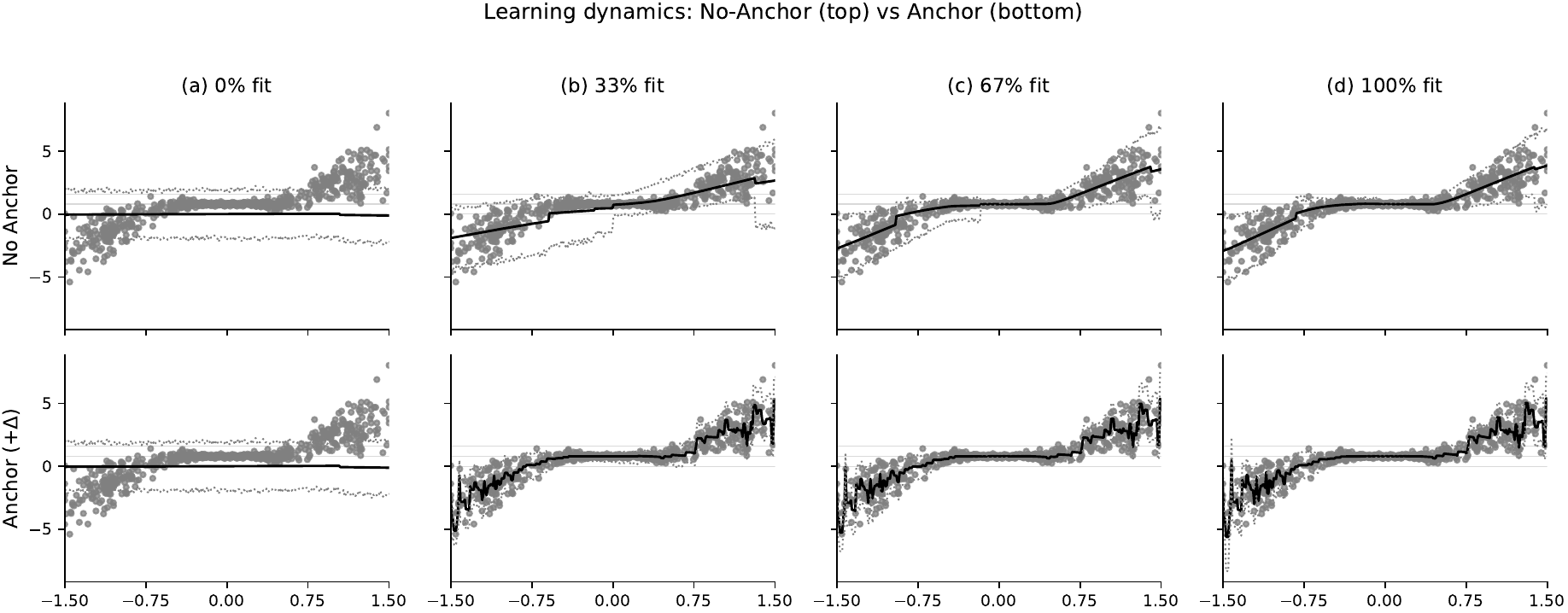}
\caption{Learning dynamics on a toy 1D dataset: No-Anchor (top) vs Anchor (+$\Delta$, bottom) at 0\%, 33\%, 67\%, and 100\% fit. Line as in Figure~1. Without anchor, updates emphasize global trend and show larger oscillations with tail variance inflation; with anchor, updates are balanced, the central plateau is preserved earlier, and predictive intervals are better calibrated.}
\label{fig:dynamics}
\end{figure}

\begin{table}[H]
  \centering
  \caption{Comparison on UCI Benchmark dataset as measured by NLL while ablating key components of Anchor-MoE. Bolding is as in Table 1.}
  \label{tab:uci_ablation_nll}
  \begin{tabular}{l c c c c c}
    \toprule
    Dataset & \textit{N} & Anchor-MoE & Anchor & Router & Calibration \\
    \midrule
    Boston  & 506  & 0.60 $\pm$ 0.11 & 0.83 $\pm$ 0.24 & \textbf{0.51 $\pm$ 0.05} & 0.52 $\pm$ 0.05\\
    
    Concrete  & 1030 & 0.25 $\pm$ 0.06 & 0.73 $\pm$ 0.04 & \textbf{0.20 $\pm$ 0.05} & \textbf{0.20 $\pm$ 0.06}\\

    Energy & 768 & \textbf{-1.68 $\pm$ 0.2} & -1.30 $\pm$ 0.05 & -0.76 $\pm$ 0.05 & -0.96 $\pm$ 0.05\\
    
    Kin8nm & 8192& \textbf{0.12 $\pm$ 0.01}  & 0.68 $\pm$ 0.02 & 1.00 $\pm$ 0.01&0.97$\pm$ 0.01\\
    
    Naval & 11934& \textbf{-1.26 $\pm$ 0.02} &  -1.09 $\pm$ 0.02  & -1.10 $\pm$ 0.02 & -1.12$\pm$ 0.02\\
    
    Power & 9568 & -0.15 $\pm$ 0.02 & -0.05 $\pm$ 0.03 & -0.15 $\pm$ 0.02& \textbf{-0.18 $\pm$ 0.02}\\
    
    Protein & 10000 & 1.06$\pm$ 0.04 & \textbf{0.63 $\pm$ 0.01}&1.05 $\pm$ 0.02 & 0.90 $\pm$ 0.03 \\
    
    Wine & 1599 & \textbf{1.20$\pm$ 0.02}  & 1.52 $\pm$ 0.43 &1.16 $\pm$ 0.02 & 1.21$\pm$0.03 \\
    
    Yacht & 308 &  -1.80$\pm$ 0.04  &  0.24 $\pm$ 0.42 & -1.76$\pm$ 0.03 & \textbf{-1.83 $\pm$ 0.03}\\
    \bottomrule
  \end{tabular}
\end{table}

\begin{table}[H]
  \centering
  \caption{Comparison on UCI Benchmark dataset as measured by RMSE while ablating key components of Anchor-MoE. Bolding is as in Table 1. Calibration can reduces RMSE significantly on Energy dataset, although it slightly increase RMSE on others.}
  \label{tab:uci_ablation_rmse}
  \begin{tabular}{l c c c c c}
    \toprule
    Dataset & \textit{N} & Anchor-MoE & Anchor & Router & Calibration \\
    \midrule
    Boston  & 506  & 3.01 $\pm$ 0.14 & 4.14 $\pm$ 0.28 & 2.88 $\pm$ 0.12 & 2.75 $\pm$ 0.10\\
    
    Concrete  & 1030 & 4.45 $\pm$ 0.16 & 7.75 $\pm$ 0.15 & 4.44 $\pm$ 0.14 & 4.18 $\pm$ 0.12\\

    Energy & 768 & 0.47 $\pm$ 0.02 & 1.48 $\pm$ 0.13 & 1.23 $\pm$ 0.04 & 1.01 $\pm$ 0.03\\
    
    Kin8nm & 8192& 0.07 $\pm$ 0.00  & 0.11 $\pm$ 0.00 & 0.15 $\pm$ 0.00& 0.15$\pm$ 0.00\\
    
    Naval & 11934& 0.00 $\pm$ 0.00 &  0.00 $\pm$ 0.00  & 0.00 $\pm$ 0.00 & 0.00$\pm$ 0.00\\
    
    Power & 9568 & 3.21 $\pm$ 0.05 & 4.01 $\pm$ 0.04 & 3.22 $\pm$ 0.05& 3.16 $\pm$ 0.05\\
    
    Protein & 10000 & 4.41$\pm$ 0.02 & 4.71 $\pm$ 0.03 &4.42 $\pm$ 0.03 &4.37 $\pm$ 0.02 \\
    
    Wine & 1599 & 0.62$\pm$ 0.01  & 0.65 $\pm$ 0.01 &0.62 $\pm$ 0.00 & 0.61$\pm$ 0.00\\
    
    Yacht & 308 & 0.62$\pm$ 0.06 &  4.19$\pm$ 0.33 & 0.62 $\pm$ 0.04 & 0.52 $\pm$ 0.04\\
    \bottomrule
  \end{tabular}
\end{table}

\section*{6. Conclusion}

We introduced \emph{Anchor--MoE}, a simple and modular architecture for point and probabilistic prediction. A small tree-based model supplies an \emph{anchor mean}, a lightweight metric–window with a soft top-$k$ router sparsely dispatches inputs to mixture-density experts, and a one–dimensional post-hoc calibrator adjusts the mean for improved point accuracy. The components are loosely coupled, easy to ablate, and the design extends naturally to classification or survival settings by swapping the likelihood.

Empirically, Anchor--MoE delivers competitive performance across standard UCI benchmarks. For probabilistic regression (test NLL), it matches or surpasses strong baselines on most datasets, with two notable exceptions (Kin8nm and Naval). For point prediction (test RMSE), Anchor--MoE and NGBoost excel on different data regimes; overall their RMSEs are closely matched and often differ only marginally.

On the theory side, under isotropic H\"older smoothness of order~$\alpha$ and fixed Lipschitz partition-of-unity weights with bounded overlap, Anchor--MoE attains the minimax-optimal $L^2$ approximation rate $\mathcal{O}\!\big(N^{-2\alpha/(2\alpha+d)}\big)$, with intrinsic-dimension refinements on manifolds of dimension $d_0$ ($\mathcal{O}\!\big(N^{-2\alpha/(2\alpha+d_0)}\big)$) and under coordinate sparsity $s$ ($\mathcal{O}\!\big(N^{-2\alpha/(2\alpha+s)}\big)$). In addition, the CRPS generalization gap scales as $\widetilde{\mathcal{O}}\!\big(\sqrt{(\log(Mh)+P+K)/N}\big)$—logarithmic in the expert width $h$ and mixture size $M$, and $\sqrt{\cdot/N}$ in the router size $P$ and number of experts $K$ (under bounded-overlap routing, $K$ can be replaced by the constant $k$). Under uniformly bounded means and variances, the same $\widetilde{\mathcal{O}}\!\big(\sqrt{(\log(Mh)+P+K)/N}\big)$ scaling holds for test NLL up to constants.

Ablations clarify the role of each component. The anchor and MoE complement each other: without an anchor, the MoE can reduce NLL by inflating variance; the anchor stabilizes variance and lets experts focus on heteroscedastic residuals. The router’s soft gating encourages specialization and balanced load; removing it leaves the window kernel to gate alone, which consistently worsens NLL and tail coverage and can harm RMSE on complex datasets. The linear mean calibrator reduces bias and improves RMSE with negligible effect on NLL in the $z$-scored space (occasionally a slight NLL increase on some datasets).

Avoiding large-variance hedging when training MoE without an anchor is a key open issue and would further decouple the two stages. Replacing the held-out mean calibration with “calibration-by-design” could simplify the pipeline and reduce data fragmentation. Exploring capacity-controlled learned routers with explicit sparsity/entropy regularization, adaptive $k$ for top-$k$ gating, and data-dependent temperature schedules may improve specialization without hurting generalization. Under covariate shift, combining Anchor--MoE with conformal calibration, latent-space OOD detection, or test time variance adaptation is a promising direction.

\section*{Acknowledgment}
This work is partially done at Linkconn Electronics Co., Ltd, we thank them for supporting this research.

\section*{Appendix}

\begin{table*}[ht]
\centering
\small
\caption{Test NLL on UCI datasets. Anchor--MoE numbers are from our runs; the other baselines are taken from prior reports of~\cite{gal2016dropout, lakshminarayanan2017simple, gal2017concrete},    . Best per row in \textbf{bold}. Protein dataset is removed as it is resampled in this study.}
\scalebox{0.8}{
\begin{tabular}{l r rrrrrrr}
\toprule
Dataset & $N$ &
\multicolumn{1}{c}{Anchor--MoE} &
\multicolumn{1}{c}{MC dropout} &
\multicolumn{1}{c}{Deep Ensembles} &
\multicolumn{1}{c}{Concrete Dropout} &
\multicolumn{1}{c}{Gaussian Process} &
\multicolumn{1}{c}{GAMLSS} &
\multicolumn{1}{c}{DistForest} \\
\midrule
Boston   &  506  & \textbf{0.60 $\pm$ 0.11} & 2.46 $\pm$ 0.25 & 2.41 $\pm$ 0.25 & 2.72 $\pm$ 0.01 & 2.37 $\pm$ 0.24 & 2.73 $\pm$ 0.56 & 2.67 $\pm$ 0.08 \\
Concrete & 1030  & \textbf{0.25 $\pm$ 0.06} & 3.04 $\pm$ 0.09 & 3.06 $\pm$ 0.18 & 3.51 $\pm$ 0.00 & 3.03 $\pm$ 0.11 & 3.24 $\pm$ 0.08 & 3.38 $\pm$ 0.05 \\
Energy   &  768  & \textbf{-1.68 $\pm$ 0.20} & 1.99 $\pm$ 0.09 & 1.38 $\pm$ 0.22 & 2.30 $\pm$ 0.00 & 0.66 $\pm$ 0.17 & 1.24 $\pm$ 0.86 & 1.53 $\pm$ 0.14 \\
Kin8nm   & 8192  & 0.12 $\pm$ 0.01 & -0.95 $\pm$ 0.03 & \textbf{-1.20 $\pm$ 0.02} & -0.65 $\pm$ 0.00 & -0.11 $\pm$ 0.03 & -0.26 $\pm$ 0.02 & -0.40 $\pm$ 0.01 \\
Naval    & 11934 & -1.26 $\pm$ 0.02 & -3.80 $\pm$ 0.05 & -5.63 $\pm$ 0.05 & \textbf{-5.87 $\pm$ 0.05} & -0.98 $\pm$ 0.02 & -5.56 $\pm$ 0.07 & -4.84 $\pm$ 0.01 \\
Power    &  9568 & \textbf{-0.15 $\pm$ 0.02} & 2.80 $\pm$ 0.05 & 2.79 $\pm$ 0.04 & 2.75 $\pm$ 0.01 & 3.81 $\pm$ 0.05 & 2.86 $\pm$ 0.04 & 2.68 $\pm$ 0.05 \\
Wine     &  1599 & 1.20 $\pm$ 0.02 & \textbf{0.93 $\pm$ 0.06} & 0.94 $\pm$ 0.12 & 1.70 $\pm$ 0.00 & 0.95 $\pm$ 0.06 & 0.97 $\pm$ 0.09 & 1.05 $\pm$ 0.15 \\
Yacht    &   308 & \textbf{-1.80 $\pm$ 0.04} & 1.55 $\pm$ 0.12 & 1.18 $\pm$ 0.21 & 1.75 $\pm$ 0.00 & 0.10 $\pm$ 0.26 & 0.80 $\pm$ 0.56 & 2.94 $\pm$ 0.09 \\
\bottomrule
\end{tabular}
}
\label{tab:moe_others}
\end{table*}


\addcontentsline{toc}{section}{Appendix}

\subsection*{A1. Minimax–optimal rate of Anchor–MoE (no dimension reduction)}
\label{app:minimax_main}

\paragraph{Notation.}
For \(d\in\N\) let \(\cF_\alpha(L)\) be the isotropic H\"older ball
of order \(\alpha>0\) and radius \(L>0\) on \([0,1]^d\) \citep[Def.~24.1]{vanderVaart1998}.
We write \(\|\cdot\|_2\) for the \(L^2([0,1]^d)\) norm and
\(\Rad_N(\cH)\) for the empirical Rademacher complexity \citep[Ch.~11]{AnthonyBartlett1999}.
Let the lattice mesh be \(h:=K^{-1/d}\).

\paragraph{Predictor and risk.}
The model is probabilistic (MDN). We evaluate the risk of the
\emph{predictive mean}. Let
\[
\widehat f_{K,N}(x)\;:=\;\E_{\widehat p(y\mid x)}[\,Y\,]
\]
be the mean of the learned predictive density \(\widehat p(y\mid x)\).
All bounds below concern \(\widehat f_{K,N}\).

\paragraph{Problem setup.}
Observe i.i.d.\ \((X_i,Y_i)\) with \(X_i\sim{\rm Unif}[0,1]^d\) and
\(Y_i=f^\star(X_i)+\varepsilon_i\) where \(\varepsilon_i\sim\cN(0,\sigma^2)\) and \(f^\star\in\cF_\alpha(L)\).
We analyse the integrated squared risk
\(\cR_N=\E\bigl[\|\,\widehat f_{K,N}-f^\star\,\|_2^2\bigr]\).

\paragraph{Model class (theoretical abstraction).}
The practical anchor mean can be absorbed into experts’ mean functions without changing rates.
We consider
\[
\cH_K
=\Bigl\{\,x\mapsto \sum_{j=1}^K w_j(x)\,e_j(x)\;:\;
\{w_j\}\ \text{is a PoU on }[0,1]^d,\ \ e_j\in\cE\Bigr\},
\]
where \(e_j(\cdot)\) denotes the \emph{expert mean function} and \(\cE\) is a bounded–capacity MDN mean class (fixed across \(K\)).

\paragraph{Assumptions.}
\begin{itemize}[leftmargin=2em]
\item[(A1)] \textbf{No dimension reduction.}
\(f_\phi=\mathrm{Id}\) on \([0,1]^d\); equivalently one may allow an invertible affine map \(f_\phi(x)=Ax+b\) with bounded condition
number, which only rescales constants.

\item[(A2)] \textbf{Partition of unity (PoU) with bounded overlap.}
Let \(\{x_j\}_{j=1}^K\) be a regular lattice with mesh \(h=K^{-1/d}\).
There exists a compactly supported PoU \(\{w_j\}_{j=1}^K\) (e.g., tensor-product B-splines)
such that \(w_j\ge0\), \(\sum_j w_j(x)=1\) for all \(x\),
\(\operatorname{diam}(\operatorname{supp} w_j)\lesssim h\), and at most \(k\) of the \(w_j(x)\) are nonzero for any \(x\)
(bounded overlap). At the boundary, cells are truncated and weights renormalized.

\item[(A3)] \textbf{Experts of bounded capacity.}
Each expert mean \(e_j\in\cE\) has fixed complexity \(\mathrm{comp}\) independent of \(K\)
(e.g., uniform Lipschitz/covering numbers or pseudo-dimension bounds; MDN variances are bounded
away from \(0\) and \(\infty\) so training is well-conditioned).
\end{itemize}

\subsubsection*{A1.1\quad Information–theoretic lower bound}

\begin{lemma}[Minimax lower bound]\label{lem:lower}
For any estimator \(\widehat f_N\) based on \(N\) samples,
\[
  \sup_{f^\star\in\cF_\alpha(L)}
  \E\bigl[\|\,\widehat f_N-f^\star\,\|_2^2\bigr]
  \;\ge\;
  C_0\,N^{-2\alpha/(2\alpha+d)} .
\]
\end{lemma}

\begin{proof}[Proof sketch]
By the metric entropy of \(\cF_\alpha(L)\),
\(\log N(\varepsilon,\cF_\alpha(L),\|\cdot\|_2)\asymp\varepsilon^{-d/\alpha}\)
\citep[Thm.~24.4]{vanderVaart1998}. A standard Fano/Assouad argument yields the rate with
\(C_0=C_0(L,\alpha,d)>0\).
\end{proof}

\subsubsection*{A1.2\quad Approximation by local interpolation (PoU)}
Let \(\{x_j\}_{j=1}^K\) be as in (A2). Define
\[
\widetilde f_K(x)
:=\sum_{j=1}^K w_j(x)\,f^\star(x_j).
\]

\begin{lemma}[Interpolation error]\label{lem:interp}
Under \emph{(A2)}, for \(f^\star\in\cF_\alpha(L)\),
\[
  \bigl\|\,\widetilde f_K - f^\star\,\bigr\|_2 \;\le\; C_1\,h^{\alpha}
  \;=\; C_1\,K^{-\alpha/d},
\]
hence \(\bigl\|\,\widetilde f_K - f^\star\,\bigr\|_2^2 = \mathcal{O}\!\bigl(K^{-2\alpha/d}\bigr)\).
\end{lemma}

\begin{proof}[Proof sketch]
On each cell, \(|f^\star(x)-f^\star(x_j)|\le L\,\|x-x_j\|^\alpha\lesssim L\,h^\alpha\).
Because \(\sum_j w_j=1\) and the overlap is uniformly bounded by \(k\),
integration over \([0,1]^d\) yields the claim (the overlap constant is absorbed into \(C_1\)).
\end{proof}

\subsubsection*{A1.3\quad Estimation error (safe form)}
\begin{lemma}[Estimation error — safe form]\label{lem:est_safe}
Under \emph{(A2)–(A3)} with overlap \(k\) and per–expert complexity \(\mathrm{comp}\) (both independent of \(K\)),
there exists \(C>0\) (depending on \(k,\mathrm{comp}\) but not on \(K,N\)) such that
\[
  \E\bigl[\|\,\widehat f_{K,N}-\widetilde f_K\,\|_2^2\bigr]
  \;\le\; C\,\frac{k\,\mathrm{comp}\,K}{N}.
\]
\end{lemma}

\begin{proof}[Proof sketch]
For \(\cH_K=\{x\mapsto \sum_{j=1}^K w_j(x)e_j(x)\}\),
bounded overlap implies
\[
\Rad_N(\cH_K)
\;\le\; \frac{1}{N}\sum_{j=1}^K
\E_\sigma\!\Big[\sup_{e_j\in\cE}\sum_{i=1}^N \sigma_i\,w_j(x_i)\,e_j(x_i)\Big]
\;\lesssim\;
\sqrt{\frac{k\,\mathrm{comp}\,K}{N}}.
\]
A standard contraction/ERM argument turns this into the stated squared error bound.
\end{proof}

\subsubsection*{A1.4\quad NLL--\(L^2\) link for Gaussian experts}
We justify evaluating the \(L^2\) risk of the predictive mean when training with Gaussian NLL.

\begin{lemma}[NLL--\(L^2\) link]\label{lem:nll_l2_link}
Assume the predictive density is Gaussian with mean \(m(x)\) and variance \(\sigma^2(x)\in[\underline\sigma^2,\overline\sigma^2]\)
(known or estimated), and let \(f^\star\) be the regression function.
Then there exist constants \(c_1,c_2>0\) depending only on \((\underline\sigma,\overline\sigma)\) such that
\[
\begin{aligned}
\mathrm{ExcessNLL}
&:= \E\bigl[-\log p_{m,\sigma}(Y\mid X)\bigr]
   - \E\bigl[-\log p_{f^\star,\sigma}(Y\mid X)\bigr] \\
&\le\
c_1\,\E\bigl[(m(X)-f^\star(X))^2\bigr] \;+\; c_2\,\E\bigl[(\sigma(X)-\sigma^\star(X))^2\bigr],
\end{aligned}
\]
where \(\sigma^\star\) is any target variance proxy. In particular, for fixed \(\sigma(\cdot)\) the excess NLL is equivalent to the \(L^2\) error of the mean up to constants.
\end{lemma}

\noindent\emph{Remark.}
If one prefers to avoid the variance term, the theory can equivalently be formulated with an ERM on the experts’ mean heads under squared loss; this leaves the implementation unchanged and only simplifies the analysis.

\subsubsection*{A1.5\quad Main bound (balancing interpolation and estimation)}
\begin{theorem}[Main bound]\label{thm:main_safe}
Under \emph{(A1)–(A3)},
\[
  \E\bigl[\|\,\widehat f_{K,N}-f^\star\,\|_2^2\bigr]
  \;\le\;
  C_1\,K^{-2\alpha/d}
  \;+\;
  C_2\,\frac{k\,\mathrm{comp}\,K}{N}.
\]
Choosing \(K^\star \asymp N^{d/(2\alpha+d)}\) yields
\[
  \sup_{f^\star\in\cF_\alpha(L)}
  \E\bigl[\|\,\widehat f_{K,N}-f^\star\,\|_2^2\bigr]
  \;\lesssim\;
  N^{-2\alpha/(2\alpha+d)},
\]
matching the minimax lower bound (Lemma~\ref{lem:lower}) up to constants.
\end{theorem}

\subsubsection*{A1.6\quad Remarks}
\begin{itemize}[leftmargin=2em]
\item[(i)] \textbf{Anchors.} The baseline “anchor” mean can be folded into expert means; it does not affect rates.
\item[(ii)] \textbf{When a \(\log K\) estimation term is valid.}
If window locations/bandwidths are fixed (non-learned), per-point aggregation uses a fixed top-\(k\) rule,
and strong parameter sharing makes the \emph{effective} number of free parameters independent of \(K\),
Lemma~\ref{lem:est_safe} can be refined to
\(\E\|\,\widehat f_{K,N}-\widetilde f_K\,\|_2^2\ \lesssim\ \tfrac{\log K + \mathrm{comp}}{N}\).
Without these structural constraints, the \(\mathcal{O}(K/N)\) bound is recommended.
\item[(iii)] \textbf{Target standardization.} Z-scoring \(Y\) only rescales constants in \(\cR_N\).
\end{itemize}

\subsection*{A2. Generalisation Bound}
       
\label{sec:genbound}

We study the population–empirical gap under the CRPS loss.
For a predictive density \(p_{\theta,\phi}(\cdot\mid x)\) define
\[
  \ell\!\bigl(p_{\theta,\phi}(\cdot\mid x),y\bigr)
  := \mathrm{CRPS}\bigl(p_{\theta,\phi},y\bigr),
  \qquad
  \cR(\theta,\phi)
  := \mathbb{E}_{(x,y)\sim\cD}\!\bigl[\ell(p_{\theta,\phi},y)\bigr],
\]
and its empirical version
\[
  \hat\cR_{N}(\theta,\phi)
  := \frac{1}{N}\sum_{i=1}^{N}
     \ell\!\bigl(p_{\theta,\phi},y_i\bigr).
\]

\paragraph{Assumptions.}
\begin{enumerate}[label=(G\arabic*),itemsep=.4ex,leftmargin=2em]
\item \textbf{(CRPS regularity and boundedness).}
  With the standard definition
  \(\mathrm{CRPS}(F,y)=\int_{\mathbb{R}}\!\bigl(F(z)-\mathbf{1}\{z\ge y\}\bigr)^2\,dz\),
  the map \(F\mapsto \mathrm{CRPS}(F,y)\) is \emph{2-Lipschitz} under the \(L^1\) metric on CDFs.
  Assume expert means are uniformly bounded \(|e_j(x)|\le R_f\) and
  the predictive variance satisfies \(\sigma(x)\in[\underline\sigma,\overline\sigma]\),
  and \(y\in[-R_y,R_y]\) almost surely (otherwise clip \(y\)).
  Then the loss is bounded by
  \[
    B \;\le\; R_f + R_y + \sqrt{\tfrac{2}{\pi}}\,\overline\sigma \, .
  \]
\item \textbf{(Model capacity).}
  For the MDN expert class \(\mathcal{H}_{M,h}\) (mixture size \(M\), width \(h\)),
  \(
    \Rad_N(\mathcal{H}_{M,h})
    \le C_h\sqrt{\frac{\log(Mh)}{N}}.
  \)
  For the router class \(\mathcal{G}_{P,K}\) with \(P\) parameters and softmax width \(K\),
  \(
    \Rad_N(\mathcal{G}_{P,K})
    \le C_g\sqrt{\frac{P+K}{N}}.
  \)
  (If the router’s final weight matrix is fully counted in \(P\), the extra “\(+K\)” can be omitted.)
\end{enumerate}

\paragraph{Composite complexity and contraction.}
Let \(\Fclass\) denote the induced class of predictive CDFs/densities
parameterised by \((K,M,h,P)\).
By the standard contraction inequality,
\begin{align}
  \Rad_N\!\bigl(\ell\!\circ\!\Fclass\bigr)
  &\le 2\,\Rad_N(\Fclass) \\
  &\le 2\,C_{\!*}\,
      \sqrt{\frac{\log(Mh)+P+K}{N}},
  \qquad
  C_{\!*}:=\max\{C_h,C_g\}\le C_h+C_g .
\end{align}

\begin{theorem}[Generalisation bound for Anchor--MoE]
\label{thm:gen}
Let \((\hat\theta,\hat\phi)\) be the parameters obtained after training on \(N\) samples.
Under \textnormal{(G1)}–\textnormal{(G2)}, for any \(\delta\in(0,1)\), with probability at least \(1-\delta\),
\begin{align}
  \cR(\hat\theta,\hat\phi)-\hat\cR_{N}(\hat\theta,\hat\phi)
  &\le
  2\,\Rad_N\!\bigl(\ell\!\circ\!\Fclass\bigr)
  \;+\;
  3B\,\sqrt{\frac{\log(2/\delta)}{2N}} \\
  &\le
  4\,\Rad_N(\Fclass)
  \;+\;
  3B\,\sqrt{\frac{\log(2/\delta)}{2N}}
  \;=\;
  \wtO{N^{-1/2}}.
\end{align}
\end{theorem}

\paragraph{Discussion.}
The bound scales as
\[
  \wtO\!\Bigl(\sqrt{(\log(Mh)+P+K)/N}\Bigr),
\]
i.e.\ logarithmic in \(Mh\) and \(\sqrt{\cdot/N}\) in \(P\) and \(K\).
Under a top-\(k\) bounded-overlap gating (each input activates at most a constant
number \(k\) of experts), the dependence on \(K\) can be replaced by \(k\).

\providecommand{\cM}{\mathcal{M}}

\subsection*{A3. High-dimensional scaling}
\label{app:high-dim}

We show that Anchor--MoE enjoys intrinsic-dimension scaling in two common high-dimensional regimes:
(i) data supported on a low-dimensional manifold; (ii) sparse coordinate dependence. In both cases the ambient
dimension \(d\) disappears from the rate, which depends only on the intrinsic dimension \(d_0\) (or sparsity \(s\)).

\paragraph{Setting A (low-dimensional manifold).}
Let \(\cM\subset[0,1]^d\) be a compact \(C^1\) submanifold of intrinsic dimension \(d_0\) and positive reach.
Let \(\mu_{\cM}\) be the normalised \(d_0\)-dimensional volume (Hausdorff) measure on \(\cM\), and interpret \(L^2(\cM)\)
with respect to \(\mu_{\cM}\). We write \(X\sim\mu_{\cM}\) (instead of \(\mathrm{Unif}(\cM)\)).
Assume \(Y=f^\star(X)+\varepsilon\) with \(\varepsilon\sim\cN(0,\sigma^2)\) and \(f^\star\in\cF_\alpha(L;\cM)\), the isotropic H\"older ball on \(\cM\).
Let \(\{w_j\}_{j=1}^K\) be a \emph{fixed (non-learned)} geodesic partition of unity (PoU) on \(\cM\) with mesh size \(h\) and bounded overlap \(k\),
so that \(\operatorname{diam}(\mathrm{supp}\,w_j)\lesssim h\) and at most \(k\) weights are nonzero at any \(x\in\cM\).
Experts have bounded capacity as in (A3) of Section~A1.

\begin{theorem}[Manifold rate]\label{thm:manifold-rate}
There exist constants \(C_1,C_2>0\) (depending only on \(L,\alpha\), the curvature/geometry of \(\cM\), the overlap \(k\),
and expert capacity) such that the predictive mean
\(\widehat f_{K,N}(x)=\sum_{j=1}^K w_j(x)\,e_j(x)\) satisfies
\[
  \E\!\left[\|\,\widehat f_{K,N}-f^\star\,\|_{L^2(\cM)}^2\right]
  \;\le\;
  C_1\,K^{-2\alpha/d_0}
  \;+\;
  C_2\,\frac{k\,\mathrm{comp}\,K}{N}.
\]
Choosing \(K^\star \asymp N^{d_0/(2\alpha+d_0)}\) yields
\[
  \sup_{f^\star\in \cF_\alpha(L;\cM)}
  \E\!\left[\|\,\widehat f_{K,N}-f^\star\,\|_{L^2(\cM)}^2\right]
  \;\lesssim\;
  N^{-2\alpha/(2\alpha+d_0)} .
\]
\end{theorem}

\begin{proof}[Sketch]
Geodesic covering numbers on \(\cM\) scale as \(h^{-d_0}\), hence \(K\asymp h^{-d_0}\).
Local H\"older interpolation on each chart gives \(\|\,\widetilde f_K-f^\star\|_{L^2(\cM)}^2 \lesssim h^{2\alpha}=K^{-2\alpha/d_0}\),
mirroring Lemma~\ref{lem:interp} with \(d\) replaced by \(d_0\).
Bounded overlap and fixed-capacity experts yield the estimation term \(C\,k\,\mathrm{comp}\,K/N\)
as in Lemma~\ref{lem:est_safe}. Balancing the two terms gives the rate.
\end{proof}

\paragraph{Setting B (sparse coordinate dependence).}
Assume there exists \(S\subset\{1,\dots,d\}\) with \(|S|=s\ll d\) such that \(f^\star(x)=g^\star(x_S)\).
Suppose the PoU \(\{w_j\}\) and gating are functions of \(x_S\) (or of a representation bi-Lipschitz in \(x_S\)),
and experts have bounded capacity. Here \(L^2\) is with respect to the marginal law of \(X\);
if the marginal density of \(X_S\) is bounded above/below on \([0,1]^s\), all constants depend only on these bounds.
The theorem below is an \emph{oracle} bound (the index set \(S\) is assumed known). If \(S\) is unknown and must be learned,
an additional model-selection penalty of order \(\wtO{(s\log d)/N}\) typically appears in the estimation term.

\begin{theorem}[Sparse rate]\label{thm:sparse-rate}
Under the sparse dependence assumption,
\[
  \E\!\left[\|\,\widehat f_{K,N}-f^\star\,\|_{L^2}^2\right]
  \;\le\;
  C_1\,K^{-2\alpha/s}
  \;+\;
  C_2\,\frac{k\,\mathrm{comp}\,K}{N},
\]
so that with \(K^\star \asymp N^{s/(2\alpha+s)}\),
\[
  \sup_{f^\star}\ \E\!\left[\|\,\widehat f_{K,N}-f^\star\,\|_{L^2}^2\right]
  \;\lesssim\;
  N^{-2\alpha/(2\alpha+s)} .
\]
\end{theorem}

\begin{proof}[Sketch]
Construct the PoU and local interpolation on the \(s\)-dimensional coordinate subspace.
Then \(K\asymp h^{-s}\) and \(\|\,\widetilde f_K-f^\star\|_2^2\lesssim h^{2\alpha}=K^{-2\alpha/s}\).
The estimation term follows as in Lemma~\ref{lem:est_safe}.
\end{proof}

\paragraph{Bi-Lipschitz invariance (optional).}
We record stability under bi-Lipschitz reparameterisations, which only rescales constants.

\begin{lemma}[Change of variables under bi-Lipschitz maps]\label{lem:bi-lip}
Let \(T:U\to V\) be bi-Lipschitz on a \(d_0\)-dimensional domain \(U\) with constants
\(a\le \|T(x)-T(x')\|/\|x-x'\|\le b\).
There exist constants \(c_1,c_2>0\) depending only on \(a,b,d_0\) such that, for any \(g,h:V\to\R\),
\[
  c_1\,\|\,g-h\,\|_{L^2(V)}
  \;\le\;
  \|(g-h)\circ T\|_{L^2(U)}
  \;\le\;
  c_2\,\|\,g-h\,\|_{L^2(V)} ,
\]
and
\(
  [\,g\circ T\,]_{C^\alpha(U)} \;\lesssim\; b^\alpha\, [\,g\,]_{C^\alpha(V)}.
\)
Positive reach of \(\cM\) yields uniformly bi-Lipschitz charts and a bounded-overlap geodesic covering; hence covering numbers scale as \(h^{-d_0}\) and Jacobian distortions are absorbed into constants (as in Lemma~\ref{lem:est_safe}, since the overlap \(k\) is constant and expert capacity is fixed).
\end{lemma}

\paragraph{Remarks.}
(i) The generalisation bound of Section~\ref{sec:genbound} scales as
\(\wtO{\sqrt{(\log(Mh)+P+K)/N}}\).
Under bounded-overlap/top-\(k\) gating (each input activates at most \(k\) experts),
the \(K\)-dependence in the complexity term can be replaced by \(k\) (a constant). \\
(ii) The balancing choices are \(K^\star \asymp N^{d_0/(2\alpha+d_0)}\) (manifold)
and \(K^\star \asymp N^{s/(2\alpha+s)}\) (sparse), offering practical guidance for coarse model selection.


\bibliographystyle{unsrtnat} 
\bibliography{references}  

\end{document}